\documentclass{llncs}
\usepackage{llncsdoc}
\usepackage{amsmath}
\usepackage{amsfonts}
\usepackage[dvipsnames]{xcolor}
\usepackage{graphicx}
\usepackage{marvosym}
\usepackage{multirow}
\usepackage[hidelinks]{hyperref}
\usepackage[font=small,format=plain,labelfont=bf,up,textfont=small,up,justification=justified,singlelinecheck=false]{caption}
\newcommand\Tstrut{\rule{0pt}{2.6ex}}
\title{Measuring Relations between Concepts in Conceptual Spaces\thanks{The final publication is available at Springer via \url{https://doi.org/10.1007/978-3-319-71078-5_7}}}
\author{Lucas Bechberger \thanks{Corresponding author; ORCID: 0000-0002-1962-1777} \and Kai-Uwe K\"uhnberger}
\institute{Institute of Cognitive Science, Osnabr\"uck University, Osnabr\"uck, Germany \email{lucas.bechberger@uni-osnabrueck.de}, \email{kai-uwe.kuehnberger@uni-osnabrueck.de}}

\begin{document}
\maketitle

%================================================================================================================================================================%

\begin{abstract}
The highly influential framework of conceptual spaces provides a geometric way of representing knowledge. Instances are represented by points in a high-dimensional space and concepts are represented by regions in this space. Our recent mathematical formalization of this framework is capable of representing correlations between different domains in a geometric way. In this paper, we extend our formalization by providing quantitative mathematical definitions for the notions of concept size, subsethood, implication, similarity, and betweenness. This considerably increases the representational power of our formalization by introducing measurable ways of describing relations between concepts.
\end{abstract}

%================================================================================================================================================================%
\section{Introduction}
\label{Intro}

One common criticism of symbolic AI approaches is that the symbols they operate on do not contain any meaning: For the system, they are just arbitrary tokens that can be manipulated in some way. This lack of inherent meaning in abstract symbols is called the “symbol grounding problem” \cite{Harnad1990}. One approach towards solving this problem is to devise a grounding mechanism that connects  abstract symbols to the real world, i.e., to perception and action.

The framework of conceptual spaces \cite{Gardenfors2000,Gardenfors2014} attempts to bridge this gap between symbolic and subsymbolic AI by proposing an intermediate conceptual layer based on geometric representations.
A conceptual space is a high-dimensional space spanned by a number of quality dimensions that are based on perception and/or subsymbolic processing. Regions in this space correspond to concepts and can be referred to as abstract symbols.

The framework of conceptual spaces has been highly influential in the last 15 years within cognitive science and cognitive linguistics \cite{Douven2011,Fiorini2013,Warglien2012}. It has also sparked considerable research in various subfields of artificial intelligence, ranging from robotics and computer vision \cite{Chella2001,Chella2003} over the semantic web \cite{Adams2009a} to plausible reasoning \cite{Derrac2015,Schockaert2011}.

One important aspect of conceptual representations is however often ignored by these research efforts: Typically, the different features of a concept are correlated with each other. For instance, there is an obvious correlation between the color and the taste of an apple: Red apples tend to be sweet and green apples tend to be sour. Recently, we have proposed a formalization of the conceptual spaces framework that is capable of representing such correlations in a geometric way \cite{Bechberger2017KI}. Our formalization not only contains a parametric definition of concepts, but also different operations to create new concepts from old ones (namely: intersection, union, and projection). 

In this paper, we provide mathematical definitions for the notions of concept size, subsethood, implication, similarity, and betweenness. This considerably increases the representational power of our formalization by introducing measurable ways of describing relations between concepts.

The remainder of this paper is structured as follows:
Section \ref{CS} introduces the general framework of conceptual spaces along with our recent formalization. In Section \ref{Extension}, we extend this formalization with additional operations and in Section \ref{Example} we provide an illustrative example. Section \ref{RelatedWork} contains a summary of related work and Section \ref{Conclusion} concludes the paper.

%================================================================================================================================================================%
\section{Conceptual Spaces}
\label{CS}

This section presents the cognitive framework of conceptual spaces as described in \cite{Gardenfors2000} and as formalized in \cite{Bechberger2017KI}.

%--------------------------------------------------------------------------------------------------------------------------------------------------------------------------------------------------------------------------------------------------------------------------------------------------------------%
\subsection{Dimensions, Domains, and Distance}
\label{CS:DimensionsDomainsDistance}

%dimensions
A conceptual space is a high-dimensional space spanned by a set $D$ of so-called ``quality dimensions''. Each of these dimensions $d \in D$ represents a way in which two stimuli can be judged to be similar or different. Examples for quality dimensions include temperature, weight, time, pitch, and hue. The distance between two points $x$ and $y$ with respect to a dimension $d$ is denoted as $|x_d - y_d|$.

%domains
A domain $\delta \subseteq D$ is a set of dimensions that inherently belong together. Different perceptual modalities (like color, shape, or taste) are represented by different domains. The color domain for instance consists of the three dimensions hue, saturation, and brightness. Distance within a domain $\delta$ is measured by the weighted Euclidean metric $d_E$.  

The overall conceptual space $CS$ is defined as the product space of all dimensions. Distance within the overall conceptual space is measured by the weighted Manhattan metric $d_M$ of the intra-domain distances. This is supported by both psychological evidence \cite{Attneave1950,Johannesson2001,Shepard1964} and mathematical considerations \cite{Aggarwal2001}. Let $\Delta$ be the set of all domains in $CS$. The combined distance $d_C^{\Delta}$ within $CS$ is defined as follows:
$$
d_C^{\Delta}(x,y,W) = \sum_{\delta \in \Delta}w_{\delta} \cdot \sqrt{\sum_{d \in \delta} w_{d} \cdot |x_{d} - y_{d}|^2}
$$
The parameter $W = \langle W_{\Delta},\{W_{\delta}\}_{\delta \in \Delta}\rangle$ contains two parts: $W_{\Delta}$ is the set of positive domain weights $w_{\delta}$ with $\textstyle\sum_{\delta \in \Delta} w_{\delta} = |\Delta|$. Moreover, $W$ contains for each domain $\delta \in \Delta$ a set $W_{\delta}$ of dimension weights $w_{d}$ with $\textstyle\sum_{d \in \delta} w_{d} = 1$.\\

%similarity as distance-based with exponential decay
The similarity of two points in a conceptual space is inversely related to their distance. This can be written as follows :
$$Sim(x,y) = e^{-c \cdot d(x,y)}\quad \text{with a constant}\; c >0 \; \text{and a given metric}\; d$$

%betweenness and convexity (formal)
Betweenness is a logical predicate $B(x,y,z)$ that is true if and only if $y$ is considered to be between $x$ and $z$. It can be defined based on a given metric $d$: 
$$B_d(x,y,z) :\iff d(x,y) + d(y,z) = d(x,z)$$

The betweenness relation based on $d_E$ results in the line segment connecting the points $x$ and $z$, whereas the betweenness relation based on $d_M$ results in an axis-parallel cuboid between the points $x$ and $z$.
One can define convexity and star-shapedness based on the notion of betweenness:

\begin{definition}
\label{def:Convexity}
(Convexity)\\
A set $C \subseteq CS$ is \emph{convex} under a metric $d \;:\iff$

\hspace{1cm}$\forall {x \in C, z \in C, y \in CS}: \left(B_d(x,y,z) \rightarrow y \in C\right)$
\end{definition}

\begin{definition}
\label{def:StarShapedSet}
(Star-shapedness)\\
A set $S \subseteq CS$ is \emph{star-shaped} under a metric $d$ with respect to a set $P \subseteq S \;:\iff$ 

\hspace{1cm}$\forall {p \in P, z \in S, y \in CS}: \left(B_d(p,y,z) \rightarrow y \in S\right)$
\end{definition}

%--------------------------------------------------------------------------------------------------------------------------------------------------------------------------------------------------------------------------------------------------------------------------------------------------------------%
\subsection{Properties and Concepts}
\label{CS:PropertiesConcepts}

G\"{a}rdenfors \cite{Gardenfors2000} distinguishes properties like ``red'', ``round'', and ``sweet'' from full-fleshed concepts like ``apple'' or ``dog'' by observing that properties can be defined on individual domains (e.g., color, shape, taste), whereas full-fleshed concepts involve multiple domains.
Each domain involved in representing a concept has a certain importance, which is reflected by so-called ``salience weights''. Another important aspect of concepts are the correlations between the different domains, which are important for both learning \cite{Billman1996} and reasoning \cite[Ch 8]{Murphy2002}.

Based on the principle of cognitive economy, G\"{a}rdenfors argues that both properties and concepts should be represented as convex sets. However, as we demonstrated in \cite{Bechberger2017KI}, one cannot geometrically encode correlations between domains when using convex sets:
The left part of Figure \ref{fig:ConvexityProblem} shows two domains, age and height, and the concepts of child and adult. The solid ellipses illustrate the intuitive way of defining these concepts. As domains are combined with the Manhattan metric, a convex set corresponds in this case to an axis-parallel cuboid. One can easily see that this convex representation (dashed rectangles) is not satisfactory, because the correlation of the two domains is not encoded. We therefore proposed in \cite{Bechberger2017KI} to relax the convexity criterion and to use star-shaped sets, which is illustrated in the right part of Figure \ref{fig:ConvexityProblem}. This enables a geometric representation of correlations while still being only a minimal departure from the original framework.\\
\begin{figure}[tp]
\centering
\includegraphics[width=0.75\columnwidth]{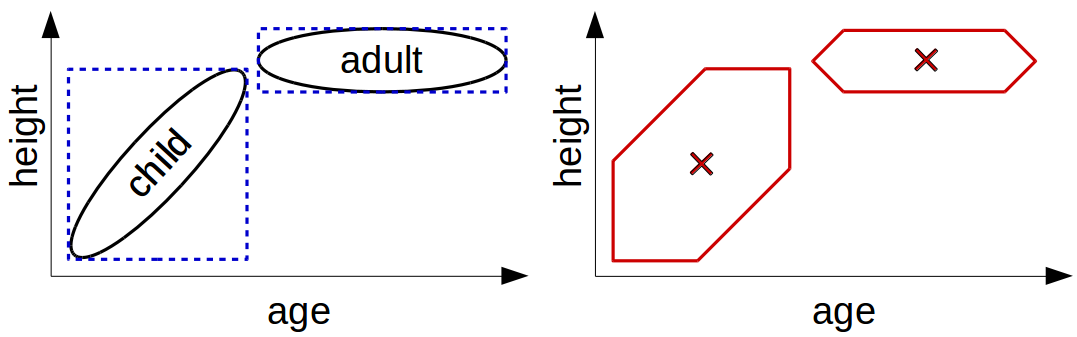}
\caption{Left: Intuitive way to define regions for the concepts of ``adult'' and ``child'' (solid) as well as representation by using convex sets (dashed). Right: Representation by using star-shaped sets with central points marked by crosses.}
\label{fig:ConvexityProblem}
\end{figure}

We have based our formalization on axis-parallel cuboids that can be described by a triple $\langle \Delta_C, p^-, p^+ \rangle$ consisting of a set of domains $\Delta_C$ on which this cuboid $C$ is defined and two points $p^-$ and $p^+$, such that 
$$x \in C \iff \forall{\delta \in \Delta_C}:\forall{d \in \delta}: p_d^- \leq x_d \leq p_d^+$$
These cuboids are convex under $d_C^{\Delta}$. It is also easy to see that any union of convex sets that have a non-empty intersection is star-shaped \cite{Smith1968}. We define the core of a concept as follows:

\begin{definition}
\label{def:SSSS}
(Simple star-shaped set)\\
A \emph{simple star-shaped set} $S$ is described as a tuple $\langle\Delta_S,\{C_1,\dots,C_m\}\rangle$. $\Delta_S \subseteq \Delta$ is a set of domains on which the cuboids $\{C_1,\dots,C_m\}$ (and thus also $S$) are defined. Moreover, it is required that the central region $P :=\textstyle\bigcap_{i = 1}^m C_i \neq \emptyset$. Then the simple star-shaped set $S$ is defined as 
$$S := \bigcup_{i=1}^m C_i$$
\end{definition}

In order to represent imprecise concept boundaries, we use fuzzy sets \cite{Bvelohlavek2011,Zadeh1965,Zadeh1982}. A fuzzy set is characterized by its membership function $\mu: CS \rightarrow [0,1]$ that assigns a degree of membership to each point in the conceptual space. The membership of a point to a fuzzy concept is based on its maximal similarity to any of the points in the concept's core:

\begin{definition}
\label{def:FSSSS}
(Fuzzy simple star-shaped set)\\
A \emph{fuzzy simple star-shaped set} $\widetilde{S}$ is described by a quadruple $\langle S,\mu_0,c,W\rangle$ where
$S = \langle\Delta_S,\{C_1,\dots,C_m\}\rangle$ is a non-empty simple star-shaped set. The parameter $\mu_0 \in (0,1]$ controls the highest possible membership to $\widetilde{S}$ and is usually set to 1. The sensitivity parameter $c > 0$ controls the rate of the exponential decay in the similarity function. Finally, $W = \langle W_{\Delta_S},\{W_{\delta}\}_{\delta \in \Delta_S}\rangle$ contains positive weights for all domains in $\Delta_S$ and all dimensions within these domains, reflecting their respective importance. We require that $\textstyle\sum_{\delta \in \Delta_S} w_{\delta} = |\Delta_S|$ and that $\forall {\delta \in \Delta_S}:\textstyle\sum_{d \in \delta} w_{d} = 1$.
The membership function of $\widetilde{S}$ is then defined as follows:
$$\mu_{\widetilde{S}}(x) = \mu_0 \cdot \max_{y \in S}(e^{-c \cdot d_C^{\Delta}(x,y,W)})$$
\end{definition}

The sensitivity parameter $c$ controls the overall degree of fuzziness of $\widetilde{S}$ by determining how fast the membership drops to zero. The weights $W$ represent not only the relative importance of the respective domain or dimension for the represented concept, but they also influence the relative fuzziness with respect to this domain or dimension.
Note that if $|\Delta_S| = 1$, then $\widetilde{S}$ represents a property, and if $|\Delta_S| > 1$, then $\widetilde{S}$ represents a concept.
Figure \ref{fig:FSSSS} illustrates this definition (the $x$ and $y$ axes are assumed to belong to different domains and are combined with $d_M$ using equal weights).
\begin{figure}[tp]
\centering
\includegraphics[width = \columnwidth]{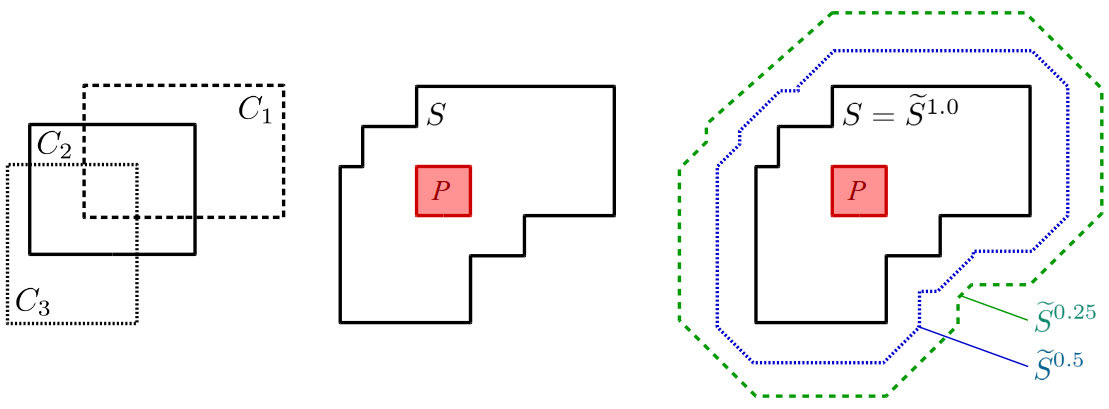} 
\caption{Left: Three cuboids $C_1, C_2, C_3$ with nonempty intersection. Middle: Resulting simple star-shaped set $S$ based on these cuboids. Right: Fuzzy simple star-shaped set $\tilde{S}$ based on $S$ with three $\alpha$-cuts for $\alpha \in \{1.0,0.5,0.25\}$.}
\label{fig:FSSSS}
\end{figure}

In our previous work \cite{Bechberger2017KI}, we have also provided a number of operations, which can be used to create new concepts from old ones: The intersection of two concepts can be interpreted as the logical ``and'' -- e.g., intersecting the property ``green'' with the concept ``banana'' results in the set of all objects that are both green and bananas. The union of two concepts can be used to construct more abstract categories (e.g., defining ``fruit'' as the union of ``apple'', ``banana'', ``coconut'', etc.). Projecting a concept onto a subspace corresponds to focusing on certain domains while completely ignoring others.

%================================================================================================================================================================%
\section{Defining Additional Operations}
\label{Extension}

\subsection{Concept Size}
\label{Extension:Hypervolume}
The size of a concept gives an intuition about its specificity: Large concepts are more general and small concepts are more specific. This is one obvious aspect in which one can compare two concepts to each other.

One can use a measure $M$ to describe the size of a fuzzy set. It can be defined in our context as follows (cf. \cite{Bouchon-Meunier1996}):
\begin{definition}
A measure $M$ on a conceptual space $CS$ is a function $M: \mathcal{F}(CS) \rightarrow \mathbb{R}^+_0$ with $M(\emptyset) = 0$ and $\widetilde{A} \subseteq \widetilde{B} \Rightarrow M(\widetilde{A}) \leq M(\widetilde{B})$, where $\mathcal{F}(CS)$ is the fuzzy power set of $CS$ and where $\widetilde{A} \subseteq \widetilde{B} :\iff \forall {x \in CS}: \mu_{\widetilde{A}}(x) \leq \mu_{\widetilde{B}}(x)$.
\end{definition}

A common measure for fuzzy sets is the integral over the set's membership function, which is equivalent to the Lebesgue integral over the fuzzy set's $\alpha$-cuts\footnote{The $\alpha$-cut of a fuzzy set $\widetilde{A}$ is defined as $\widetilde{A}^{\alpha} = \{ x \in CS \;|\; \mu_{\widetilde{A}}(x) \geq \alpha\}$.}:
\begin{equation}
M(\widetilde{A}) := \int_{CS} \mu_{\widetilde{A}}(x)\; dx = \int_{0}^1 V(\widetilde{A}^{\alpha})\; d\alpha
\label{eqn:integral}
\end{equation}
We use $V(\widetilde{A}^{\alpha})$ to denote the volume of a fuzzy set's $\alpha$-cut.
Let us define for each cuboid $C_i \in S$ its fuzzified version $\widetilde{C}_i$ as follows (cf. Definition \ref{def:FSSSS}):
$$\mu_{\widetilde{C}_i}(x) = \mu_0 \cdot \max_{y \in C_i}(e^{-c \cdot d_C^{\Delta}(x,y,W)})$$
It is obvious that $\mu_{\widetilde{S}}(x) = \max_{C_i \in S} \mu_{\widetilde{C}_i}(x)$. It is also clear that the intersection of two fuzzified cuboids is again a fuzzified cuboid. Finally, one can easily see that we can use the inclusion-exclustion formula (cf. e.g., \cite{Bogart1989}) to compute the overall measure of $\widetilde{S}$ based on the measure of its fuzzified cuboids:
\begin{equation}
M(\widetilde{S}) = \sum_{l=1}^m \left((-1)^{l+1} \cdot \sum_{\substack{\{i_1,\dots,i_l\}\\\subseteq\{1,\dots,m\}}}M\left(\bigcap_{i \in \{i_1,\dots,i_l\}} \widetilde{C}_i\right)\right)
\label{eqn:inclusionExclusion}
\end{equation}
The outer sum iterates over the number of cuboids under consideration (with $m$ being the total number of cuboids in S) and the inner sum iterates over all sets of exactly $l$ cuboids. The overall formula generalizes the observation that $|\widetilde{C}_1 \cup \widetilde{C}_2| = |\widetilde{C}_1| + |\widetilde{C}_2| - |\widetilde{C}_1 \cap \widetilde{C}_2|$ from two to $m$ sets.

In order to derive $M(\widetilde{C})$, we first describe how to compute $V(\widetilde{C}^{\alpha})$, i.e., the size of a fuzzified cuboid's $\alpha$-cut. Using Equation \ref{eqn:integral}, we can then derive $M(\widetilde{C})$, which we can in turn insert into Equation \ref{eqn:inclusionExclusion} to compute the overall size of $\widetilde{S}$.

Figure \ref{fig:2DAlphaCut} illustrates the $\alpha$-cut of a fuzzified two-dimensional cuboid both under $d_E$ (left) and under $d_M$ (right). Because the membership function is defined based on an exponential decay, one can interpret each $\alpha$-cut as an $\epsilon$-neighborhood of the original cuboid C, where $\epsilon$ depends on $\alpha$: 
$$
x \in {\widetilde{C}}^{\alpha} \iff \mu_0 \cdot \max_{y \in C}(e^{-c \cdot d_C^{\Delta}(x,y,W)}) \geq \alpha \iff  \min_{y \in C} d_C^{\Delta}(x,y,W) \leq -\frac{1}{c} \cdot \ln(\frac{\alpha}{\mu_0})
$$
\begin{figure}[tp]
\centering
\includegraphics[width = 0.8\textwidth]{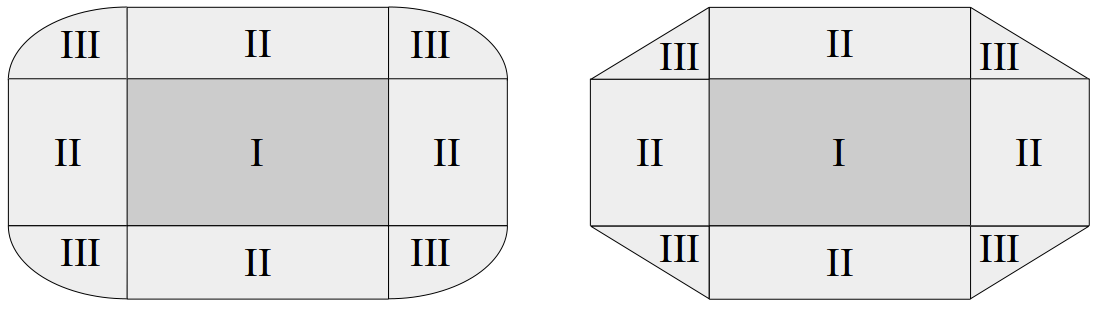}
\caption{$\alpha$-cut of a fuzzified cuboid under $d_E$ (left) and $d_M	$ (right), respectively.}
\label{fig:2DAlphaCut}
\end{figure}

$V(\widetilde{C}^{\alpha})$ can be described as a sum of different components. Let us use the shorthand notation $b_d := p_d^+ - p_d^-$.
Looking at Figure \ref{fig:2DAlphaCut}, one can see that all components of $V(\widetilde{C}^{\alpha})$ can be described by ellipses\footnote{Note that ellipses under $d_M$ have the form of streched diamonds.}: Component I is a zero-dimensional ellipse (i.e., a point) that was extruded in two dimensions with extrusion lengths of $b_1$ and $b_2$, respectively. Component II consists of two one-dimensional ellipses (i.e., line segments) that were extruded in one dimension, and component III is a two-dimensional ellipse. 

Let us denote by $\Delta_{\{d_1,\dots,d_i\}}$ the domain structure obtained by eliminating from $\Delta$ all dimensions $d \in D\setminus\{d_1, \dots, d_i\}$. Moreover, let $V(r, \Delta, W)$ be the hypervolume of a hyperball under $d_C^\Delta(\cdot,\cdot, W)$ with radius $r$. In this case, a hyperball is the set of all points with a distance of at most $r$ (measured by $d_C^\Delta(\cdot,\cdot, W)$) to a central point. Note that the weights $W$ will cause this ball to have the form of an ellipse. For instance, in Figure \ref{fig:2DAlphaCut}, we assume that $w_{d_1} < w_{d_2}$ which means that we allow larger differences with respect to $d_1$ than with respect to $d_2$. This causes the hyperballs to be streched in the $d_1$ dimension, thus obtaining the shape of an ellipse.
We can in general describe $V(\widetilde{C}^{\alpha})$ as follows:
$$V(\widetilde{C}^{\alpha}) = \sum_{i=0}^n \left( \sum_{\substack{\{d_1,\dots,d_i\}\\ \subseteq D}} \left( \prod_{\substack{d \in\\ D\setminus\{d_1,\dots,d_i\}}} b_d \right) \cdot V\left( -\frac{1}{c} \cdot \ln\left(\frac{\alpha}{\mu_0}\right),\Delta_{\{d_1, \dots, d_i\}},W \right)\right)$$
The first sum of this formula runs over the number of dimensions with respect to which a given point $x \in \widetilde{C}^{\alpha}$ lies outside of $C$. We then sum over all combinations $\{d_1,\dots,d_i\}$ of dimensions for which this could be the case, compute the volume of the $i$-dimensional hyperball under these dimensions ($V(\cdot,\cdot,\cdot)$) and extrude this intermediate result in all remaining dimensions ($\prod_{d \in D\setminus\{d_1,\dots,d_i\}} b_d$).

Let us illustrate this formula for the $\alpha$-cuts shown in Figure \ref{fig:2DAlphaCut}: For $i = 0$, we can only select the empty set for the inner sum, so we end up with $b_1 \cdot b_2$, which is the size of the original cuboid (i.e., component I). For $i = 1$, we can either pick $\{d_1\}$ or $\{d_2\}$ in the inner sum. For $\{d_1\}$, we compute the size of the left and right part of component II by multiplying $V\left( -\frac{1}{c} \cdot \ln\left(\frac{\alpha}{\mu_0}\right),\Delta_{\{d_1\}},W \right)$ (i.e., their combined width) with $b_2$ (i.e., their height). For $\{d_2\}$, we analogously compute the size of the upper and the lower part of component II. Finally, for $i = 2$, we can only pick $\{d_1, d_2\}$ in the inner sum, leaving us with $V\left( -\frac{1}{c} \cdot \ln\left(\frac{\alpha}{\mu_0}\right),\Delta ,W \right)$, which is the size of component III.
One can easily see that the formula for $V(\widetilde{C}^{\alpha})$ also generalizes to higher dimensions.

\begin{proposition}
\label{proposition:hyperball}
$V(r,\Delta, W) = \frac{1}{\prod_{\delta \in \Delta} w_{\delta} \cdot \prod_{d \in \delta} \sqrt{w_d}} \cdot \frac{r^n}{n!} \cdot \prod_{\delta \in \Delta} \left(|\delta|! \cdot \frac{\pi^{\frac{|\delta|}{2}}}{\Gamma(\frac{|\delta|}{2}+1)}\right)$
\end{proposition}
\begin{proof}
See appendix (\url{http://lucas-bechberger.de/appendix-sgai-2017/}).
\end{proof}
%As we have shown in \cite{Bechberger2017Hyperball}, $V(r, \Delta, W)$ can be computed as follows:
%$$V(r,\Delta, W) = \frac{1}{\prod_{\delta \in \Delta} w_{\delta} \cdot \prod_{d \in \delta} \sqrt{w_d}} \cdot \frac{r^n}{n!} \cdot \prod_{\delta \in \Delta} \left(|\delta|! \cdot \frac{\pi^{\frac{|\delta|}{2}}}{\Gamma(\frac{|\delta|}{2}+1)}\right)$$
%
Defining $\delta(d)$ as the unique $\delta \in \Delta$ with $d \in \delta$, and $a_d := w_{\delta(d)} \cdot \sqrt{w_{d}} \cdot b_d \cdot c$, we can use Proposition \ref{proposition:hyperball} to rewrite $V(\widetilde{C}^{\alpha})$:

\begin{align*}
V(\widetilde{C}^\alpha) &=  
\frac{1}{c^n\prod_{d \in D} w_{\delta(d)} \sqrt{w_d}}
\sum_{i=0}^{n} \Bigg( 
\frac{(-1)^i \cdot \ln\left(\frac{\alpha}{\mu_0}\right)^i}{i!} \cdot 
\sum_{\substack{\{d_1,\dots,d_i\}\\ \subseteq D}} 
\left(\prod_{\substack{d \in \\D \setminus \{d_1,\dots,d_i\}}} a_d\right) \cdot \\
&\hspace{4.5cm}\prod_{\substack{\delta \in \\ \Delta_{\{d_1,\dots,d_i\}}}} \left(
|\delta|! \cdot \frac{\pi^{\frac{|\delta|}{2}}}{\Gamma(\frac{|\delta|}{2}+1)}\right)\Bigg)\\
\end{align*}
We can solve Equation \ref{eqn:integral} to compute $M(\widetilde{C})$ by using the following lemma:

\begin{lemma}
\label{lemma:logIntegral}
$\forall n \in \mathbb{N}: \int_0^{1} \ln(x)^n dx = (-1)^n \cdot n!$ 
\end{lemma}
\begin{proof}
Substitute $x = e^t$ and $s = -t$, then apply the definition of the $\Gamma$ function.
\end{proof}

\begin{proposition}
\label{proposition:Measure}
The measure of a fuzzified cuboid $\widetilde{C}$ can be computed as follows:
\begin{align*}
M(\widetilde{C}) &= \frac{\mu_0}{c^n\prod_{d \in D} w_{\delta(d)} \sqrt{w_d}}
\sum_{i=0}^{n} \Bigg( 
\sum_{\substack{\{d_1,\dots,d_i\}\\ \subseteq D}} 
\left(\prod_{\substack{d \in \\ D \setminus \{d_1,\dots,d_i\}}} a_d\right) \cdot\\
&\hspace{5cm}\prod_{\substack{\delta \in\\ \Delta_{\{d_1,\dots,d_i\}}}} \left(
|\delta|! \cdot \frac{\pi^{\frac{|\delta|}{2}}}{\Gamma(\frac{|\delta|}{2}+1)}\right)\Bigg)
\end{align*}
\end{proposition}
\begin{proof}
Substitute $x = \frac{\alpha}{\mu_0}$ in Equation \ref{eqn:integral} and apply Lemma \ref{lemma:logIntegral}.
\end{proof}

Note that although the formula for $M(\widetilde{C})$ is quite complex, it can be easily implemented via a set of nested loops. As mentioned earlier, we can use the result from Proposition \ref{proposition:Measure} in combination with the inclusion-exclusion formula (Equation \ref{eqn:inclusionExclusion}) to compute $M(\widetilde{S})$ for any concept $\widetilde{S}$. Also Equation \ref{eqn:inclusionExclusion} can be easily implemented via a set of nested loops.
Note that $M(\widetilde{S})$ is always computed only on $\Delta_S$, i.e., the set of domains on which $\widetilde{S}$ is defined.

\subsection{Subsethood}
\label{Extension:Subsethood}
In order to represent knowledge about a hierarchy of concepts, one needs to be able to determine whether one concept is a subset of another concept. The classic definition of subsethood for fuzzy sets reads as follows:
$$\widetilde{S}_1 \subseteq \widetilde{S}_2 :\iff \forall {x \in CS}: \mu_{\widetilde{S}_1}(x) \leq \mu_{\widetilde{S}_2}(x)$$
This definition has the weakness of only providing a binary/crisp notion of subsethood. It is desirable to define a \emph{degree} of subsethood in order to make more fine-grained distinctions. 
Many of the definitions for degrees of subsethood proposed in the fuzzy set literature \cite{Bouchon-Meunier1996,Young1996} require that the underlying universe is discrete. The following definition \cite{Kosko1992} works also in a continuous space and is conceptually quite straightforward: 
$$Sub(\widetilde{S}_1,\widetilde{S}_2) = \frac{M(\widetilde{S}_1 \cap \widetilde{S}_2)}{M(\widetilde{S}_1)} \quad \text{with a measure } M$$
One can interpret this definition intuitively as the ``percentage of $\widetilde{S}_1$ that is also in $\widetilde{S}_2$''. It can be easily implemented based on the measure defined in Section \ref{Extension:Hypervolume} and the intersection defined in \cite{Bechberger2017KI}. If $\widetilde{S}_1$ and $\widetilde{S}_2$ are not defined on the same domains, then we first project them onto their shared subset of domains before computing their degree of subsethood.

When computing the intersection of two concepts with different sensitivity parameters $c^{(1)}, c^{(2)}$ and different weights $W^{(1)}, W^{(2)}$, one needs to define new parameters $c'$ and $W'$ for the resulting concept. In our earlier work \cite{Bechberger2017KI}, we have argued that the sensitivity parameter $c'$ should be set to the minimum of $c^{(1)}$ and $c^{(2)}$. As a larger value for $c$ causes the membership function to drop faster, this means that the concept resulting from intersecting two imprecise concepts is at least as imprecise as the original concepts. Moreover, we defined $W'$ as a linear interpolation between $W^{(1)}$ and $W^{(2)}$. The importance of each dimension and domain to the new concept will thus lie somewhere between its importance with respect to the two original concepts.

Now if $c^{(1)} > c^{(2)}$, then $c' = \min(c^{(1)}, c^{(2)}) = c^{(2)} < c^{(1)}$. It might thus happen that $M(\widetilde{S}_1 \cap \widetilde{S}_2) > M(\widetilde{S}_1)$, and that therefore $Sub(\widetilde{S}_1,\widetilde{S}_2) > 1$. As we would like to confine $Sub(\widetilde{S}_1,\widetilde{S}_2)$ to the interval $[0,1]$, we should use the same $c$ and $W$ for computing both $M(\widetilde{S}_1 \cap \widetilde{S}_2)$ and $M(\widetilde{S}_1)$. 

When judging whether $\widetilde{S}_1$ is a subset of $\widetilde{S}_2$, we can think of $\widetilde{S}_2$ as setting the context by determining the relative importance of the different domains and dimensions as well as the degree of fuzziness. For instance, when judging whether tomatoes are vegetables, we focus our attention on the features that are crucial to the definition of the ``vegetable'' concept. We thus propose to use $c^{(2)}$ and $W^{(2)}$ when computing $M(\widetilde{S}_1 \cap \widetilde{S}_2)$ and $M(\widetilde{S}_1)$.

\subsection{Implication}
\label{Extension:Implication}

Implications play a fundamental role in rule-based systems and all approaches that use formal logics for knowledge representation. It is therefore desirable to define an implication function on concepts, such that one is able to express facts like $apple \Rightarrow red$ within our formalization.

In the fuzzy set literature \cite{Mas2007}, a fuzzy implication is defined as a generalization of the classical crisp implication. Computing the implication of two fuzzy sets typically results in a new fuzzy set which describes for each point in the space the validity of the implication. In our setting, we are however more interested in a single number that indicates the overall validity of the implication $apple \Rightarrow red$.
We propose to reuse the definition of subsethood from Section \ref{Extension:Subsethood}: It makes intuitive sense in our geometric setting to say that $apple \Rightarrow red$ is true to the degree to which $apple$ is a subset of $red$. We therefore define:
$$Impl(\widetilde{S}_1,\widetilde{S}_2) := Sub(\widetilde{S}_1,\widetilde{S}_2)$$

\subsection{Similarity and Betweenness}
\label{Extension:SimilarityBetweenness}

In our prior work \cite{Bechberger2017KI} (cf. Section \ref{CS:DimensionsDomainsDistance}), we have already provided definitions for similarity and betweenness of points. We can naively define similarity and betweenness for concepts by applying the definitions from Section \ref{CS:DimensionsDomainsDistance} to the midpoints of the concepts' central regions $P$ (cf. Definition \ref{def:SSSS}). Betweenness is a binary relation and independent of dimension weights and sensitivity parameters. For computing the similarity, we propose to use both the dimension weights and the sensitivity parameter of the second concept, which again in a sense provides the context for the similarity judgement. If the two concepts are defined on different sets of domains, we use only their common subset of domains for computing the distance of their midpoints and thus their similarity.
%================================================================================================================================================================%
\section{Illustrative Example}
\label{Example}

%---------------------------------------------------------------------------------------------------------------%
\subsection{A Conceptual Space and its Concepts}
\label{Example:Definition}

We consider a very simplified conceptual space for fruits, consisting of the following domains and dimensions:
$$\Delta = \{\delta_{color} = \{d_{hue}\},\delta_{shape} = \{d_{round}\},\delta_{taste} = \{d_{sweet}\}\}$$
$d_{hue}$ describes the hue of the observation's color, ranging from $0.00$ (purple) to $1.00$ (red). $d_{round}$ measures the percentage to which the bounding circle of an object is filled. $d_{sweet}$ represents the relative amount of sugar contained in the fruit, ranging from 0.00 (no sugar) to 1.00 (high sugar content). As all domains are one-dimensional, the dimension weights $w_{d}$ are always equal to 1.00 for all concepts. We assume that the dimensions are ordered like this: $d_{hue},d_{round},d_{sweet}$. Table \ref{tab:FruitSpace} defines several concepts in this space and Figure \ref{fig:FruitSpace} visualizes them.

\begin{table}[t]
  \centering
  \caption{Definitions of several concepts.}
  \begin{tabular}{|l||c|c|c|c|c|c|c|c|}
    \hline
    Concept 	& $\Delta_S$& $p^-$ 				& $p^+$ 				& $\mu_0$ 	& $c$ 	& \multicolumn{3}{|c|}{$W$}\\ %\hline 
    & & & & & & $w_{\delta_{color}}$ & $w_{\delta_{shape}}$ & $w_{\delta_{taste}}$\\ \hline \hline
    Orange		& $\Delta$	& (0.80, 0.90, 0.60)	& (0.90, 1.00, 0.70)	& 1.0		& 15.0	& 1.00 & 1.00 & 1.00 \\ \hline
    Lemon		& $\Delta$	& (0.70, 0.45, 0.00)	& (0.80, 0.55, 0.10)	& 1.0		& 20.0	& 0.50 & 0.50 & 2.00 \\ \hline
    Granny & \multirow{2}{*}{$\Delta$}	& \multirow{2}{*}{(0.55, 0.70, 0.35)}	& \multirow{2}{*}{(0.60, 0.80, 0.45)}	& \multirow{2}{*}{1.0}		& \multirow{2}{*}{25.0} & \multirow{2}{*}{1.00} & \multirow{2}{*}{1.00} &  \multirow{2}{*}{1.00}\\ 
    Smith & & & & & & & &\\ \hline
    \multirow{3}{*}{Apple}	& \multirow{3}{*}{$\Delta$}	& (0.50, 0.65, 0.35)	& (0.80, 0.80, 0.50)	& \multirow{3}{*}{1.0}		& \multirow{3}{*}{10.0}	& \multirow{3}{*}{0.50} & \multirow{3}{*}{1.50} &  \multirow{3}{*}{1.00} \\ %\hline
    			&			& (0.65, 0.65, 0.40)	& (0.85, 0.80, 0.55)	&			&		&  & & \\ %\hline
    			&			& (0.70, 0.65, 0.45)	& (1.00, 0.80, 0.60)	&			&		&  & & \\ \hline
    Red			& $\{\delta_{color}\}$ & (0.90, -$\infty$, -$\infty$) & (1.00, +$\infty$, +$\infty$) & 1.0 & 20.0 & 1.00 & -- & -- \\ \hline
  \end{tabular}\\[1ex]
  \label{tab:FruitSpace}
\end{table}

\begin{figure}[t]
\centering
\includegraphics[width=0.9\textwidth]{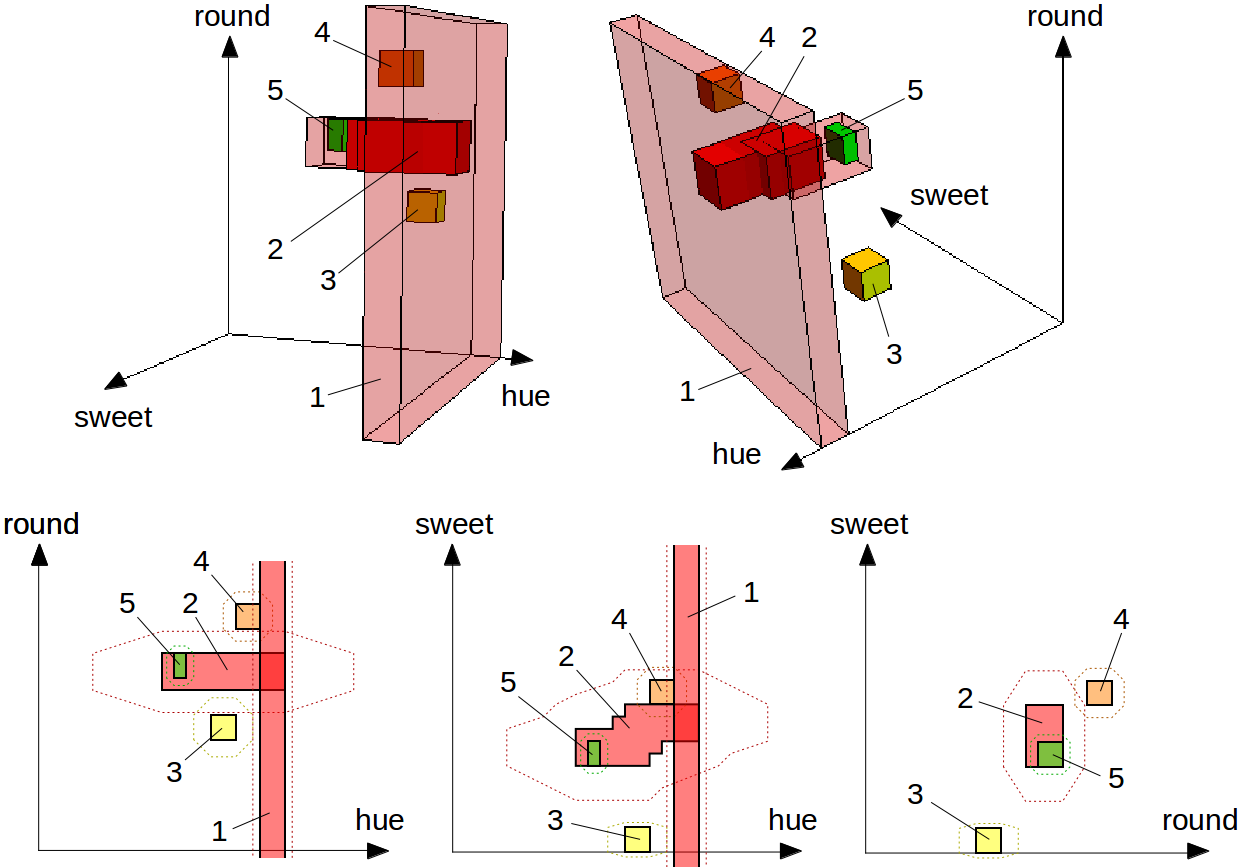}
\caption{Top: Three-dimensional visualization of the fruit space (only cores). Bottom: Two-dimensional visualizations of the fruit space (cores and 0.5-cuts). The concepts are labeled as follows: red (1), apple (2), lemon (3), orange (4), Granny Smith (5).}
\label{fig:FruitSpace}
\end{figure}

%---------------------------------------------------------------------------------------------------------------%
\subsection{Computations}
\label{Example:Computations}

\begin{table}[t]
  \centering
  \caption{Computations of different relations. Note that $Impl(\widetilde{S}_1, \widetilde{S}_2) = Sub(\widetilde{S}_1, \widetilde{S}_2)$.}
  \begin{tabular}{|c|c||c|c|c|c|c|c|}
  \hline
  $\widetilde{S}_1$ & $\widetilde{S}_2$ & $M(\widetilde{S}_1)$ & $M(\widetilde{S}_2)$ & $Sub(\widetilde{S}_1, \widetilde{S}_2)$ & $Sub(\widetilde{S}_2, \widetilde{S}_1)$ & $Sim(\widetilde{S}_1, \widetilde{S}_2)$ & $Sim(\widetilde{S}_2, \widetilde{S}_1)$\Tstrut \\ \hline \hline
  Granny Smith 	& Apple 	& 0.0042 	& 0.1048 	& 1.0000 	& 0.1171 	& 0.1353 	& 0.0010 	\\ \hline
  Orange 		& Apple 	& 0.0127	& 0.1048	& 0.1800	& 0.0333	& 0.0036	& 0.0006	\\ \hline
  Lemon			& Apple		& 0.0135	& 0.1048	& 0.0422	& 0.0054	& 0.0005	& 0.0000	\\ \hline
  Red			& Apple		& 0.2000	& 0.1048	& 1.0000	& 0.3333	& 0.3679	& 0.0183	\\ \hline
  \end{tabular}\\[2ex]
  \begin{tabular}{|c|c|c||c|}
  \hline
  $\widetilde{S}_1$ & $\widetilde{S}_2$ & $\widetilde{S}_2$ & B($\widetilde{S}_1$, $\widetilde{S}_2$, $\widetilde{S}_3)$\Tstrut \\ \hline \hline
  Lemon			& Apple			& Orange	& True 	\\ \hline
  Lemon			& Granny Smith	& Orange	& False	\\ \hline
  Granny Smith	& Apple			& Orange	& False	\\ \hline
  \end{tabular}\\[1ex]
  \label{tab:computations}
\end{table}

Table \ref{tab:computations} shows the results of using the definitions from Section \ref{Extension} on the concepts defined in Section \ref{Example:Definition}. Note that $M(\widetilde{S}_{lemon}) \neq M(\widetilde{S}_{orange})$ because the two concepts have different weights and different sensitivity parameters. Also all relations involving the property ``red'' tend to yield relatively high numbers -- this is because all computations only take place within the single domain on which ``red'' is defined. The numbers computed for the subsethood/implication relation nicely reflect our intuitive expectations. Finally, both the values for similarity and betweenness can give a rough idea about the relationship between concepts, but can only yield relatively shallow insights. This indicates that a less naive approach is needed for these two relations. Especially a fuzzy betweenness relation yielding a degree of betweenness seems to be desirable.

%================================================================================================================================================================%
\section{Related Work}
\label{RelatedWork}

Our work is of course not the first attempt to devise an implementable formalization of the conceptual spaces framework.

%Aisbett2001
An early and very thorough formalization was done by Aisbett \& Gibbon \cite{Aisbett2001}. Like we, they consider concepts to be regions in the overall conceptual space. However, they stick with the assumption of convexity and do not define concepts in a parametric way. The only operations they provide are distance and similarity of points and regions. Their formalization targets the interplay of symbols and geometric representations, but it is too abstract to be implementable. 
% dist & sim of objects & regions (Hausdorff); classification

%Rickard2006
Rickard \cite{Rickard2006} provides a formalization based on fuzziness. He represents concepts as co-occurence matrices of their properties. By using some mathematical transformations, he interprets these matrices as fuzzy sets on the universe of ordered property pairs. Operations defined on these concepts include similarity judgements between concepts and between concepts and instances. Rickard's representation of correlations is not geometrical: He first discretizes the domains (by defining properties) and then computes the co-occurences between these properties. Depending on the discretization, this might lead to a relatively coarse-grained notion of correlation. Moreover, as properties and concepts are represented in different ways, one has to use different learning and reasoning mechanisms for them. His formalization is also not easy to work with due to the complex mathematical transformations involved.
% concept similarity & observation-concept similarity; (subsethood)

%Adams2009
Adams \& Raubal \cite{Adams2009} represent concepts by one convex polytope per domain. This allows for efficient computations while being potentially more expressive than our cuboid-based representation. The Manhattan metric is used to combine different domains. However, correlations between different domains are not taken into account and cannot be expressed in this formalization as each convex polytope is only defined on a single domain. Adams \& Raubal also define operations on concepts, namely intersection, similarity computation, and concept combination. This makes their formalization quite similar in spirit to ours. 
% intersectino of regions; membership point in region; distance within a domain; instance similarity, concept similarity; concept combination (property-concept, concept-concept, relativeProperty-concept)

%Lewis2016
Lewis \& Lawry \cite{Lewis2016} formalize conceptual spaces using random set theory. They define properties as random sets within single domains and concepts as random sets in a boolean space whose dimensions indicate the presence or absence of properties. In order to define this boolean space, a single property is taken from each domain. Their approach is similar to ours in using a distance-based membership function to a set of prototypical points. However, their work purely focuses on modeling conjunctive concept combinations and does not consider correlations between domains. 
% concept combination (prop-prop); conjunction of compound concepts (property-concept, concept-concept)

As one can see, none of the formalizations listed above provides a set of operations that is as comprehensive as the one offered by our extended formalization.
%================================================================================================================================================================%
\section{Conclusion and Future Work}
\label{Conclusion}

In this paper, we extended our previous formalization of the conceptual spaces framework by providing ways to measure relations between concepts: concept size, subsethood, implication, similarity, and betweenness. This considerably extends our framework's capabilities for representing knowledge and makes it (to the best of our knowledge) the most thorough and comprehensive formalization of conceptual spaces developed so far.

In future work, we will implement this extended formalization in software. Moreover, we will provide more thorough definitions of similarity and betweenness for concepts, given that our current definitions are rather naive. A potential starting point for this can be the betwenness relations defined by Derrac \& Schockaert \cite{Derrac2015}. Finally, our overall research goal is to use machine learning in conceptual spaces, which will put this formalization to practical use.

% The file named.bst is a bibliography style file for BibTeX 0.99c
\bibliographystyle{splncs03}
\bibliography{/home/lbechberger/Documents/Papers/jabref.bib}

\end{document}